\documentclass{article} 

\usepackage{xcolor}  
\usepackage[breaklinks=true,colorlinks]{hyperref}
\hypersetup{
  linkcolor={red!50!black},
  citecolor={blue!50!black},
}   
\usepackage{url}

\usepackage[utf8]{inputenc} 
\usepackage[truedimen,margin=20mm]{geometry} 

\usepackage[T1]{fontenc}    
\usepackage{url}            
\usepackage{booktabs}       
\usepackage{amsfonts}       
\usepackage{nicefrac}       
\usepackage{microtype}      
\usepackage[pdftex]{graphicx}
\usepackage{natbib} 
\usepackage{algorithm} 
\usepackage{algorithmic} 
\usepackage{multirow}
\usepackage{amsmath}
\usepackage{amssymb}
\usepackage{amsthm}
\usepackage{here}
\usepackage{wrapfig}
\usepackage{centernot}
\usepackage{enumitem}
\usepackage{comment}
\usepackage{subfig}
\usepackage{bbm}
\usepackage{titlesec}
\usepackage{abstract} 

\allowdisplaybreaks[1]
\titleformat*{\section}{\large\bfseries}

\newtheorem{theorem}{Theorem}
\newtheorem{lemma}{Lemma}
\newtheorem{proposition}{Proposition}
\newtheorem{definition}{Definition}
\newtheorem{corollary}{Corollary}
\newtheorem{assumption}{Assumption}
\newtheorem{example}{Example}

\usepackage{color}

\newcommand{\defeq}{\overset{\mathrm{def}}{=}}
\newcommand{\disteq}{\overset{\mathrm{d}}{=}}

\def\pd<#1>{\left\langle #1 \right\rangle}
\def\floor[#1]{\left\lfloor #1 \right\rfloor}
\def\ceil[#1]{\left\lceil #1 \right\rceil}

\newcommand{\rd}{\mathrm{d}}

\newcommand{\bE}{\mathbb{E}}
\newcommand{\bN}{\mathbb{N}}

\newcommand{\bR}{\mathbb{R}}

\newcommand{\cD}{\mathcal{D}}

\newcommand{\cL}{\mathcal{L}}
\newcommand{\cN}{\mathcal{N}}
\newcommand{\cP}{\mathcal{P}}
\newcommand{\cX}{\mathcal{X}}

\newcommand{\KL}{\mathrm{KL}}

\title{Convex Analysis of the Mean Field Langevin Dynamics}

\author{Atsushi Nitanda$^{1\dag}$, Denny Wu$^{2\ddag}$, Taiji Suzuki$^{3\star}$
\vspace{2mm}\\
\normalsize{\textit{$^1$Kyushu Institute of Technology and RIKEN Center for Advanced Intelligence Project}} \\
\normalsize{\textit{$^2$The University of Toronto and Vector Institute for Artificial Intelligence}} \\
\normalsize{\textit{$^3$The University of Tokyo and RIKEN Center for Advanced Intelligence Project}} \\
\small{Email: $^\dag$nitanda@ai.kyutech.ac.jp, $^\ddag$dennywu@cs.toronto.edu, $^\star$taiji@mist.i.u-tokyo.ac.jp}} 

\date{}

\begin{document}
\twocolumn
\maketitle

\begin{abstract}
As an example of the nonlinear Fokker-Planck equation, the \textit{mean field Langevin dynamics} recently attracts attention due to its connection to (noisy) gradient descent on infinitely wide neural networks in the mean field regime, and hence the convergence property of the dynamics is of great theoretical interest. In this work, we give a concise and self-contained convergence rate analysis of the mean field Langevin dynamics with respect to the (regularized) objective function in both continuous and discrete time settings. The key ingredient of our proof is a \textit{proximal Gibbs distribution} $p_q$ associated with the dynamics, which, in combination with techniques in \cite{vempala2019rapid}, allows us to develop a simple convergence theory parallel to classical results in convex optimization. Furthermore, we reveal that $p_q$ connects to the \textit{duality gap} in the empirical risk minimization setting, which enables efficient empirical evaluation of the algorithm convergence.
\end{abstract}

\section{INTRODUCTION}

Consider a neural network with $M$ trainable neurons parameterized as
\begin{equation}\label{eq:nn} 
    h_\Theta(x) \defeq \frac{1}{M}\sum_{r=1}^M h_{\theta_r}(x),  
\end{equation}
where each neuron $h_{\theta_r}$ contains trainable parameters (weights) $\theta_r$ and some nonlinear transformation (e.g., $h_{\theta_r}(x) = \sigma(\langle\theta_r,x\rangle)$ where $\sigma$ is the nonlinear activation function), $x$ is a data example, and $\Theta = \{\theta_r\}_{r=1}^M$. Under suitable conditions, as $M\to\infty$ we obtain the \textit{mean field limit}: $h_q(x) := \bE_q[h_\theta(x)]$, where $q(\theta)\rd\theta$ represents the probability distribution of the weights; we refer to $h_q$ simply as a mean field model (neural network). 
In this limit, training can be formulated as an optimization problem over the space of probability measures. 
An advantage of the mean field regime, in contrast to alternative settings such as the neural tangent kernel regime~\citep{jacot2018neural}, is the presence of (nonlinear) \textit{feature learning} \citep{suzuki2018adaptivity,ghorbani2019limitations}. However, developing an optimization theory for mean field neural networks is also more challenging. 
 
Optimization analysis of mean field neural networks usually utilizes the convexity of the objective function in the space of probability measures. \cite{nitanda2017stochastic,chizat2018global,mei2018mean} established global convergence of gradient descent (flow) on two-layer neural networks in the mean field regime under appropriate conditions. Subsequent works proved convergence rates under additional structural assumptions \citep{javanmard2019analysis,chizat2021sparse}. 
One noticeable algorithmic modification is the addition of Gaussian noise to the gradient, which leads to the \textit{noisy} gradient descent algorithm; this modification gives rise to an entropy regularization term in the objective, and allows for global convergence analysis under less restrictive settings \citep{rotskoff2018parameters,mei2019mean}. The corresponding stochastic dynamics is often referred to as the \textit{mean field Langevin dynamics}.  

Recent works have studied the convergence rate of the mean field Langevin dynamics, including its underdamped (kinetic) version. However, most existing analyses either require sufficiently strong regularization \citep{hu2019mean,jabir2019mean}, or build upon involved mathematical tools \citep{kazeykina2020ergodicity,guillin2021kinetic}. 
Our goal is to provide a simpler convergence proof that covers general and more practical machine learning settings, with a focus on neural network optimization in the mean field regime. 
Motivated by an observation in \cite{nitanda2020particle} that the \textit{log-Sobolev inequality} can simplify the global convergence analysis of two-layer mean field neural network, we study the optimization efficiency of the mean field Langevin dynamics in the context of KL-regularized empirical/expected risk minimization, and present a new convergence rate analysis by translating the finite-dimensional convex optimization theory to the optimization in the space of measures.

\begin{figure}[!htb]  
\centering
\begin{minipage}[t]{0.98\linewidth} 
{\includegraphics[width=1\linewidth]{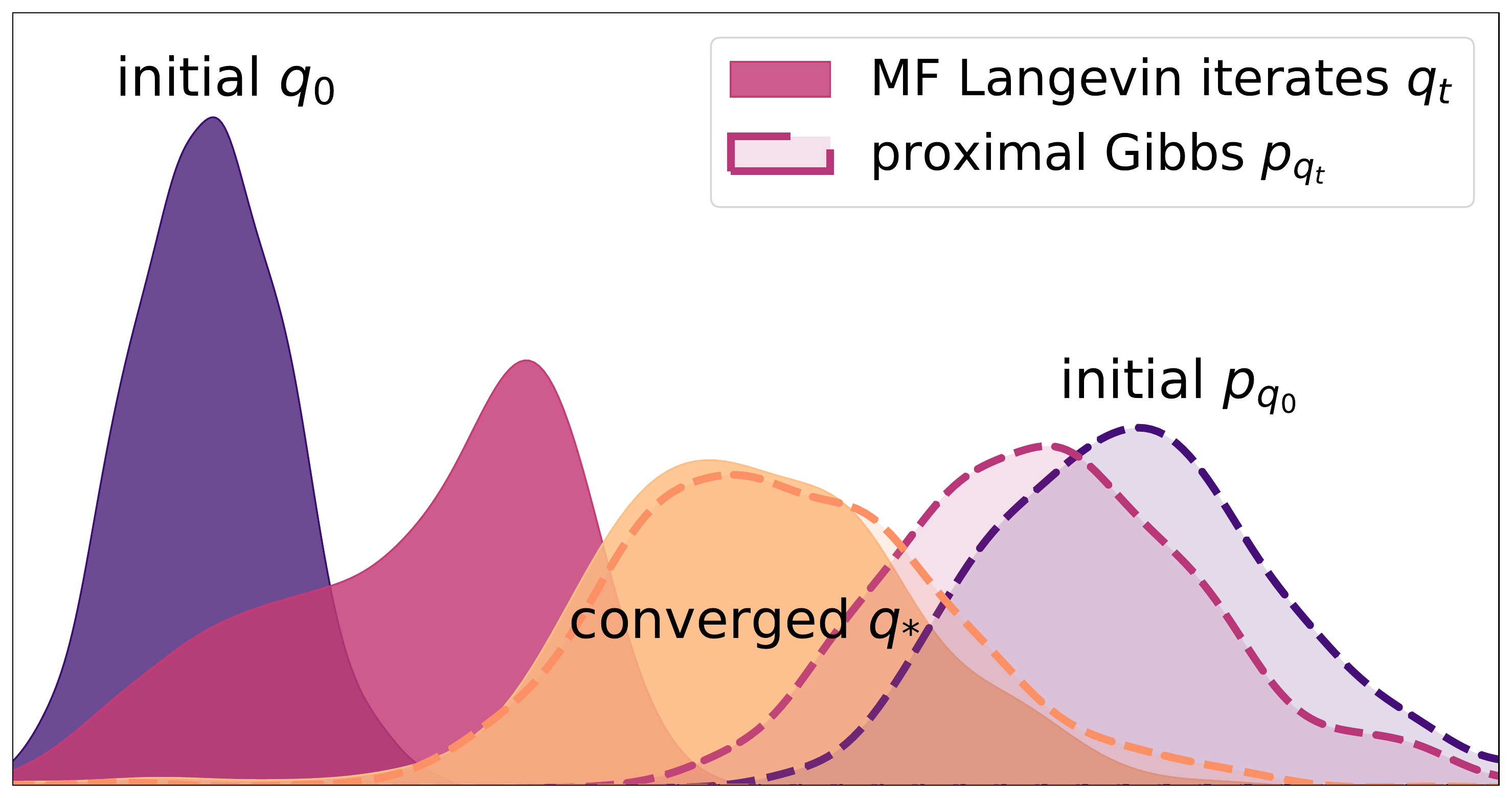}}  
\end{minipage}
\captionof{figure}{\small 1D visualization of mean field two-layer neural network (tanh) optimized by noisy gradient descent ($\lambda=\lambda'=10^{-2}$). Observe that both the parameter distribution $q$ and the corresponding proximal Gibbs distribution $p_q$ converge to the optimal $q_*$. Moreover, $q_t$ and $p_{q_t}$ approach $q_*$ from opposite directions, as predicted by Proposition~\ref{prop:optimization_gap_for_strong_convex_problems}.}   
\label{fig:pq_illustration} 
\end{figure} 

\subsection{Contributions}
 
In this work, we give a simple and self-contained convergence rate analysis of the mean field Langevin dynamics in both continuous and discrete time settings. The key ingredient of our proof is the introduction of a {\it proximal Gibbs distribution} $p_q$ (see Figure~\ref{fig:pq_illustration}), which relates to the optimization gap and allows us to directly apply standard convex optimization techniques. In particular,
\begin{itemize}[leftmargin=*,itemsep=0.1mm,topsep=0.1mm]
    \item By analyzing the proximal Gibbs distribution $p_q$, we establish linear convergence in continuous time with respect to the KL-regularized objective. 
    This convergence result holds for \textit{any} regularization parameters, in contrast to existing analyses (e.g., \cite{hu2019mean}) that require strong regularization.
    \item We also provide global convergence rate analysis for the discrete-time update. This is achieved by extending the classical ``one-step interpolation'' argument in the analysis of Langevin dynamics (e.g., see \cite{vempala2019rapid} for the KL case) to our nonlinear Fokker-Planck setting. 
    \item Finally, we present an interpretation of the proximal Gibbs distribution via the primal-dual formulation of empirical risk minimization problems, and reveal that $p_q$ exactly fills the duality gap. This interpretation leads to alternative ways to evaluate the convergence of the algorithm. 
\end{itemize}

\subsection{Related Literature}

\paragraph{Convergence of the Langevin algorithm.} 
The Langevin dynamics can be interpreted as the (Wasserstein) gradient flow of KL divergence with respect to the target distribution $p\propto\text{exp}(-f)$ \citep{jordan1998variational}; in other words, the Langevin dynamics solves the following optimization problem, 
\begin{equation}\label{eq:linear-functional}
 \min_{q: \mathrm{density}} \left\{ \bE_q[f] + \bE_q[\log(q)] \right\}. 
\end{equation}
We refer readers to \cite{dalalyan2017further,wibisono2018sampling} for additional discussions on the connection between sampling and optimization. 

It is well-known that when the target distribution satisfies certain isoperimetry conditions such as the log-Sobolev inequality, then the (continuous-time) Langevin dynamics converges exponentially \citep{bakry2013analysis}. 
On the other hand, the time-discretized dynamics, or the Langevin algorithm, admits a biased invariant distribution depending on the step size and the numerical scheme \citep{milstein2013stochastic,li2019stochastic}, and sublinear convergence rate has been established under different target assumptions \citep{vempala2019rapid,erdogdu2020convergence} and distance metrics \citep{dalalyan2014theoretical,durmus2017nonasymptotic,erdogdu2021convergence}.

\paragraph{Mean field regime and nonlinear Fokker-Planck.}
Analysis of neural networks in the mean field regime typically describes the optimization dynamics as a partial differential equation (PDE) of the parameter distribution, from which convergence to the global optimal solution may be shown \citep{chizat2018global,mei2018mean,rotskoff2018parameters,sirignano2020mean}. 
Quantitative convergence rate usually requires additional conditions, such as structural assumptions on the learning problem \citep{javanmard2019analysis,chizat2021sparse,akiyama2021}, or modification of the dynamics \citep{rotskoff2019global,wei2019regularization}.
Noticeably, \citet{mei2018mean,hu2019mean,chen2020generalized} considered the optimization of the \textit{KL-regularized objective}, which leads to the mean field Langevin dynamics. 
   
Note that the optimization of mean field neural networks (e.g., Eq.~\eqref{eq:mean-field-nn}) falls beyond the scope of Langevin dynamics whose density function follows \textit{linear} Fokker-Planck equation and solves~\eqref{eq:linear-functional}, due to nonlinear loss function. Instead, the density function of parameters follows a \textit{nonlinear} Fokker-Planck equation (Eq.~\eqref{eq:nonlinear_FP_equation}), the convergence rate of which is more difficult to establish.  
\cite{hu2019mean,jabir2019mean} obtained convergence rate of the mean field Langevin dynamics under sufficiently strong regularization. 
Exponential convergence of related dynamics (e.g., the underdamped variant and other McKean-Vlasov equations) in continuous time have been shown under various settings \citep{monmarche2017long,guillin2019uniform,kazeykina2020ergodicity,guillin2021kinetic}, based on hypocoercivity \citep{villani2009hypocoercivity} or coupling techniques \citep{eberle2019quantitative}. 
\citet{bou2020convergence,bou2021mixing} studied the discrete time convergence of Hamiltonian Monte Carlo with interaction potential. 
  
Our work builds upon \citet{nitanda2020particle}, which employs the Langevin algorithm to solve the inner loop of a \textit{dual averaging} method in the space of measures.  Importantly, under a uniform log-Sobolev inequality on the ``linearized'' objective, (sublinear) global convergence rate in minimizing an entropy-regularized nonlinear functional can be proved by adapting finite-dimensional convex optimization theory. 
Based on similar ideas combined with the primal-dual formulation of empirical risk minimization, a \textit{stochastic dual coordinate ascent} alternative has been developed in \citet{oko2022psdca} to achieve linear convergence in discrete time. 
A few recent works also considered the adaptation of classical convex optimization algorithms into the space of measures, such as the Mirror descent method \citep{ying2020mirror}, the Frank-Wolfe method \citep{kent2021frank}, and the Bregman proximal gradient method \citep{chizat2021convergence}. 

Concurrent and independent to our work, \citet{chizat2022mean} analyzed the mean field Langevin dynamics also using properties of the proximal Gibbs distribution $p_q$; while both works build upon the same observation, \citet{chizat2022mean} focused on the continuous-time convergence rate and studied the annealed dynamics, whereas we establish discrete-time guarantees for the noisy gradient descent algorithm and present a primal-dual viewpoint.
  
\vspace{-0.5mm}  
\subsection{Notations} 
\vspace{-0.25mm}
$\|\cdot\|_2$ denotes the Euclidean norm.
Given a density function $q(\theta)$ on $\bR^d$, we write the expectation w.r.t.~$q(\theta) \rd \theta$ as $\bE_{\theta \sim q}[\cdot]$ or simply $\bE_{q}[\cdot]$, $\bE_\theta[\cdot]$ when the random variable and distribution are obvious from the context; e.g.~for a function $f: \bR^d \rightarrow \bR$, we write $\bE_q[f] = \int f(\theta) q(\theta)\rd\theta$ when $f$ is integrable. 
$\KL$ stands for the Kullback-Leibler divergence:
$\KL(q\|q') \defeq \int q(\theta) \log \left( \frac{q(\theta)}{q'(\theta)}\right)\rd\theta$. 
Let $\cP$ be the space of probability density functions with respect to $\rd \theta$ on $\bR^d$ such that the entropy and second moment are well-defined. 

\section{PRELIMINARIES}
The mean field Langevin dynamics is the target of our convergence analysis. In this section, we introduce this dynamics as well as the associated optimization problem. We also outline one major application of our convergence analysis, which is the optimization of mean field neural networks.
\subsection{Problem Setup}
Let $F:\cP \to \bR$ be a differentiable convex functional. 
That is, we suppose there is a functional $\frac{\delta F}{\delta q}:~\cP \times \bR^d \ni (q,\theta) \mapsto \frac{\delta F}{\delta q}(q)(\theta) \in \bR$ such that for any $q, q' \in \cP$,
\[ \left.\frac{\rd F (q+\epsilon (q'-q))}{\rd\epsilon} \right|_{\epsilon=0} 
= \int  \frac{\delta F}{\delta q}(q)(\theta) (q'-q)(\theta) \rd\theta \]
and $F$ satisfies the convexity condition:
\begin{equation}\label{eq:convexity}
F(q') \geq F(q) + \int \frac{\delta F}{\delta q}(q)(\theta) (q'-q)(\theta) \rd \theta.      
\end{equation}

We consider the minimization of an entropy regularized nonlinear functional:
\begin{equation}\label{prob:org}
    \min_{q \in \cP} 
    \left\{ 
    \cL(q) \defeq F(q) + \lambda \bE_{q} [ \log q ]
    \right\}.
\end{equation}
For $q \in \cP$, we next define an associated Gibbs distribution which plays a key role in our analysis: the {\it proximal Gibbs distribution} around $q$ in $\cP$ (see Proposition \ref{prop:optimization_gap_for_strong_convex_problems}).
\begin{definition}[Proximal Gibbs distribution]
    We define $p_q(\theta)$ to be the Gibbs distribution with potential function $-\lambda^{-1} \delta F(q)/\delta q$, i.e., 
    \begin{equation}
        p_q(\theta) = \frac{\exp\left( - \frac{1}{\lambda} \frac{\delta F(q)}{\delta q}(\theta) \right)}{Z(q)},
    \end{equation}
    where $Z(q)$ is the normalization constant.
\end{definition}

In this study, we provide basic convergence results for the general form of the problem (\ref{prob:org}) under the following assumptions, which will be specialized and verified for mean field neural network introduced in the sequel.  
\begin{assumption}\label{assumption:potential}
    For a functional $F:\cP \to \bR$ and any $q \in \cP$, assume the functional derivative $\frac{\delta F}{\delta q}(q)(\theta)$ exists and is smooth in $\theta$. Moreover, assume $|Z(q)| < \infty$, $\frac{\delta F}{\delta q}(q)(\theta)=O(1+\|\theta\|_2^2)$ uniformly over $\cP$, and $F$ is a convex functional, that is, (\ref{eq:convexity}) holds for any $q, q'\in \cP$.
\end{assumption}
\begin{assumption}[Log-Sobolev inequality] \label{assumption:ls_ineq}
    Suppose there exists a constant $\alpha>0$ such that for any $q \in \cP$, the probability distribution $p_q(\theta)\rd \theta$ satisfies log-Sobolev inequality with constant $\alpha$, that is, for any smooth function $g: \bR^d \to \bR$, we have
    \[ \bE_{p_q}[g^2 \log g^2] - \bE_{p_q}[g^2]\log\bE_{p_q}[g^2] \leq \frac{2}{\alpha}\bE_{p_q}[\|\nabla g\|_2^2]. \]
\end{assumption}
Under Assumption \ref{assumption:ls_ineq}, by setting $g = \sqrt{q/p_q}$, we get
\begin{equation}\label{eq:ls_ineq}
    \KL( q \| p_q) \leq \frac{1}{2\alpha} \bE_{q}\left[ \left\| \nabla \log \frac{q}{p_q} \right\|_2^2 \right].    
\end{equation}

\subsection{Optimization Dynamics}
To solve the problem \eqref{prob:org}, we consider the following (continuous-time) {\it mean field Langevin dynamics}:
\begin{equation}\label{eq:MF-Langevin-Dynamics}
    \rd\theta_t = - \nabla \frac{\delta F}{\delta q}(q_t)(\theta_t) \rd t + \sqrt{2\lambda}\rd W_t,
\end{equation}
where $\theta_t \sim q_t(\theta)\rd \theta$ and $\{W_t\}_{t\geq 0}$ is the Brownian motion in $\bR^d$ with $W_0=0$.
Here, the gradient $\nabla$ of the functional derivative $\frac{\delta F(q)}{\delta q}(\theta)$ is applied with respect to $\theta$.
It is known that the distribution of $\theta_t$ following the dynamics (\ref{eq:MF-Langevin-Dynamics}) solves nonlinear Fokker–Planck equation:
\begin{equation}\label{eq:nonlinear_FP_equation}
    \frac{\partial q_t}{\partial t}
    = \nabla \cdot \left( q_t \nabla \frac{\delta F}{\delta q}(q_t) \right) 
    + \lambda \Delta q_t.
\end{equation}
We here reformulate the equation (\ref{eq:nonlinear_FP_equation}) as follows:
\begin{align} \label{eq:nonlinear_FP_equation_2}
    \frac{\partial q_t}{\partial t}
    &= \lambda \nabla \cdot \left( q_t\nabla \log \exp\left( \frac{1}{\lambda} \frac{\delta F}{\delta q}(q_t) \right)
    + q_t \nabla \log q_t \right) \notag \\
    &= \lambda \nabla \cdot \left( q_t \nabla \log \frac{q_t}{p_{q_t}}\right).
\end{align}
As we will see later, this formulation involving $p_{q_t}$ will be more useful in the convergence analysis.

Moreover, we also consider the standard discretization of the above dynamics with step size $\eta>0$:
\begin{equation}\label{eq:noisy-GD}
    \theta^{(k+1)} = \theta^{(k)} - \eta \nabla \frac{\delta F}{\delta q}(q^{(k)})(\theta^{(k)}) + \sqrt{2\lambda \eta}\xi^{(k)},
\end{equation}
where $\theta^{(k)}$ is a random variable following $q^{(k)}(\theta)\rd\theta$ and $\xi^{(k)} \sim \cN(0,I_d)$.
We present convergence rate analysis for both the continuous- and discrete-time algorithm. 

\subsection{Mean Field Neural Networks}\label{subsec:mfnn}
The main application of our theory is to provide quantitative convergence guarantees for risk minimization problems with the mean field neural network. Before formally defining the mean field limit, we first introduce the finite dimensional counterpart.
\paragraph{Finite dimensional case.}
Let $\cX$ be the data space and $h_\theta: \cX \rightarrow \bR$ be a component of neural network, which corresponds to a single neuron with trainable parameter $\theta \in \bR^d$.
Then, an $M$-neuron network $h_\Theta$ (where $\Theta=\{\theta_r\}_{r=1}^M$) is defined as the average of these components as in \eqref{eq:nn}.
Let $\ell(z,y): \bR \times \bR \to \bR$ be a smooth and convex loss function in $z$, such as the logistic loss and squared loss, and $\rho$ be the empirical or true data distribution. 
Then, the regularized risk of $h_\Theta$ is defined as 
\begin{equation}\label{eq:risk_with_finite-model}
    \bE_{(X,Y) \sim \rho}[ \ell(h_\Theta(X),Y)] + \frac{\lambda'}{M} \sum_{r=1}^M r(\theta_r),
\end{equation}
where $r(\theta)$ is a convex regularizer such as $\ell_2$-penalty $r(\theta)=\|\theta\|_2^2$, and $\lambda'>0$ is regularization strength.
Note that this formulation includes both empirical and expected risk minimization problems.
To optimize the objective (\ref{eq:risk_with_finite-model}), we perform the gradient descent update with step size $\eta M >0$:
\begin{align}\label{eq:gradient_descent}
    &g_r^{(k)} = \bE[\partial_{z}\ell(h_{\Theta^{(k)}}(X),Y) \partial_{\theta_r}h_{\theta_r^{(k)}}(X)] + \lambda' \partial_{\theta_r} r(\theta_r^{(k)}), \notag\\
    &\theta_r^{(k+1)} = \theta_r^{(k)} - \eta g_r^{(k)}.
\end{align}

\paragraph{Mean field limit.}
We now take the limit $M \to \infty$ and suppose $\theta_r$ follows a probability distribution $q(\theta)\rd\theta$.
Then, $h_\Theta(x)$ converges to the mean field limit $h_q(x) = \bE_{\theta\sim q}[h_{\theta}(x)]$ in which the density function $q$ is recognized as the parameter, and the objective \eqref{eq:risk_with_finite-model} converges to the following convex functional,
\begin{equation} 
\label{eq:mean-field-nn}
    F(q) = \bE_{(X,Y)}[ \ell(h_q(X),Y)] + \lambda' \bE_{\theta \sim q}[ r(\theta) ].
\end{equation}
In this case, the functional derivative $\frac{\delta F(q)}{\delta q}$ is 
\[ \frac{\delta F(q)}{\delta q}(\theta) = \bE_{(X,Y)}[ \partial_z\ell(h_q(X),Y) h_\theta(X)] + \lambda' r(\theta). \]
By noticing that $\nabla \frac{\delta F(q)}{\delta q}(\theta)$ is the mean field limit ($M\to \infty$) of the step $g_r^{(k)}$ for the finite-dimensional model, we see that the discrete dynamics (\ref{eq:noisy-GD}) is the noisy variant of the mean field limit of gradient descent (\ref{eq:gradient_descent}). 
This motivates us to study mean field Langevin dynamics (\ref{eq:MF-Langevin-Dynamics}) and its discretization (\ref{eq:noisy-GD}).

\paragraph{Verification of assumptions.}
Under smoothness and boundedness conditions on $h_\theta$, and also smooth convex loss function $\ell$ with $\ell_2$-penalty term $r(\theta)=\|\theta\|_2^2$, one can easily verify that objective of mean field neural network satisfies Assumption \ref{assumption:potential}. 
The convexity of $F$ immediately follows by taking the expectation $\bE_{(X,Y)}$ of the following inequality:
\begin{align*}
 &\ell(h_q(X),Y) + \int \partial_z \ell(h_q(X),Y) h_\theta(X)( q' - q)(\theta) \rd \theta \\
 &=\ell(h_q(X),Y) + \partial_z \ell(h_q(X),Y)(h_{q'}(X) - h_q(X)) \\
 &\leq \ell(h_{q'}(X),Y).
\end{align*}

Moreover, with $r(\theta)=\|\theta\|_2^2$, the uniform log-Sobolev inequality (Assumption \ref{assumption:ls_ineq}) with constant $\alpha=\frac{2\lambda'}{\lambda \exp(O(\lambda^{-1}))}$ can be verified via a standard application of the LSI perturbation lemma \citep{holley1987logarithmic} (see Appendix~\ref{app:LSI} for details), since $p_q$ is proportional to the Gibbs distribution specified as:  
\begin{equation}\label{eq:gibbs_dist}
    \!\exp\left( -\frac{1}{\lambda} \bE_{(X,Y)}[ \partial_z\ell(h_q(X),Y) h_\theta(X)] 
    - \frac{\lambda'}{\lambda} r(\theta) \right)\!.\!\!\!
\end{equation}

\section{CONVERGENCE ANALYSIS}\label{sec:convergence_analysis}
In this section, we present the convergence rate analysis of mean field Langevin dynamics in both continuous- and discrete-time settings. 
We remark that our proof strategy can be seen as a combination and extension of $(i)$ convex analysis parallel to the finite-dimensional optimization setting, and $(ii)$ convergence analysis for the linear Fokker-Planck equation (e.g., \cite{vempala2019rapid}). 
As we will see, the proximal Gibbs distribution $p_q$ plays an important role in connecting these different techniques.
\subsection{Basic Properties and Convexity}  

We first present a basic but important result that characterizes the role of $p_q$ in the convergence analysis. 
Note that the functional derivative of the negative entropy $\bE_q[\log q]$ in the density function $q$ is $\log q$\footnote{ We ignore a constant in the functional derivative because this derivative is applied to function $g: \bR^d \to \bR$ that satisfies $\int g(\theta) \rd \theta = 0$ or is differentiated by the gradient $\nabla$.}; therefore, the optimality condition of the problem (\ref{prob:org}) is $\frac{\delta \cL}{\delta q}(q) = \frac{\delta F}{\delta q}(q) + \lambda \log q = 0$. This is to say, the optimal probability density function $q_*$ satisfies 
\begin{equation}\label{eq:opt-condition}
    q_*(\theta) = p_{q_*}(\theta) \propto \exp\left( - \frac{1}{\lambda} \frac{\delta F(q_*)}{\delta q}(\theta)\right).
\end{equation} 
This fact was shown in \cite{hu2019mean}. 
Hence, we may interpret the divergence between probability density functions $q$ and $p_q$ as an \textit{optimization gap}. Indeed, this intuition is confirmed by the following proposition, which can be established using standard convex analysis.
In particular, the proof relies on the fact that the negative entropy acts as a strongly convex function with respect to $\KL$-divergence. 

\begin{proposition}\label{prop:optimization_gap_for_strong_convex_problems}~
Under Assumption \ref{assumption:potential}, we have the three following statements,
\begin{enumerate}[topsep=0.2mm] 
\setlength{\leftskip}{-3mm} 
    \item In the sense of functional on the space of probability density functions, the following equality holds.
    \begin{equation*} 
        \frac{\delta \cL}{\delta q}(q) = \lambda\frac{\delta}{\delta q'} \KL(q' \| p_q)|_{q'=q} =  \lambda \log \frac{q}{p_q}. 
    \end{equation*} 
    In other words, for any $g = p-p'$ $(p, p'\in \cP)$, 
    \begin{equation*}
        \int \frac{\delta \cL}{\delta q}(q)(\theta) g(\theta) \rd \theta 
        = \int \lambda \log \left( \frac{q}{p_q}(\theta) \right) g(\theta) \rd \theta.    
    \end{equation*}
    \item For any probability distributions $q, q' \in \cP$, we have
    \begin{equation}\label{prob:proximal-point}
        \hspace{-2mm}\cL(q) + \int \frac{\delta \cL}{\delta q}(q)(\theta)(q'-q)(\theta)\rd\theta + \lambda \KL(q'\|q) \leq \cL(q').
    \end{equation}
    Moreover, $p_q$ associated with $q \in \cP$ is a minimizer of the left hand side of this inequality in $q' \in \cP$.
    \item Let $q_*$ be an optimal solution of (\ref{prob:org}).
    Then, for any $q \in \cP$, we get
    \begin{equation*}
        \lambda \KL( q \| p_q) \geq \cL(q) - \cL(q_*) \geq \lambda \KL( q \| q_*).
    \end{equation*}
\end{enumerate}

\end{proposition}
\begin{proof}
(i) We start with the first statement.
\begin{align*}
    \frac{\delta \cL}{\delta q}(q) 
    &= \frac{\delta F}{\delta q}(q) + \lambda \log q \\
    &= - \lambda \log \exp\left( - \frac{1}{\lambda} \frac{\delta F}{\delta q}(q) \right) + \lambda \log q \\
    &= \lambda ( \log q - \log p_q ) - \lambda \log Z(p_q),
\end{align*}
where $Z(p_q)$ is a normalization constant of $p_q$.
Moreover, we can see $\frac{\delta}{\delta q'} \KL(q' \| p_q)|_{q'=q} = \log q - \log p_q$. 

(ii) From direct computation, for any $q, q' \in \cP$, 
\begin{align}\label{eq:entropy-difference}
    \bE_{q'}[\log(q')] 
    &= \bE_{q}[\log(q)] + \KL(q^{\prime}\|q) \notag\\
    &+ \int \frac{\delta}{\delta q}\bE_{q}[\log(q)](\theta) (q'-q)(\theta) \rd\theta .
\end{align}

By the convexity of $F$ and (\ref{eq:entropy-difference}), we get 
\begin{align*}
    \cL(q') 
    &= F(q') + \lambda \bE_{q'} [ \log q'] \\
    &\geq F(q) + \int \frac{\delta F}{\delta q}(q)(\theta) ( q' -q )(\theta) \rd \theta 
    + \lambda \bE_{q'} [ \log q'] \\
    &= \cL(q) + \int \frac{\delta \cL}{\delta q}(q)(\theta) ( q' -q )(\theta) \rd \theta 
    + \lambda \KL(q^{\prime}\|q).
\end{align*}

In addition, by taking the functional derivative of the left hand side of (\ref{prob:proximal-point}) in $q'$, we obtain the optimality condition of this functional as follows:
\[ 0 = \frac{\delta \cL}{\delta q} + \lambda( \log q' - \log q ) = -\lambda \log p_q + \lambda \log q'. \]
Therefore, $q' = p_q$ is a minimizer of the left hand side of (\ref{prob:proximal-point}) as desired.

(iii) For the last statement, observe that minimizing both sides of (\ref{prob:proximal-point}) over $q' \in \cP$ yields 
\begin{align*}
    \cL(q_*) 
    &\geq  \cL(q) + \int \frac{\delta \cL}{\delta q}(q)(\theta) ( p_q -q )(\theta) \rd \theta 
    + \lambda \KL(p_q\|q) \\
    &= \cL(q) + \lambda \int \log \frac{q}{p_q}(\theta) ( p_q -q )(\theta) \rd \theta + \lambda \KL(p_q\|q) \\
    &= \cL(q) - \lambda \KL( q \| p_q).
\end{align*}
Moreover, by (\ref{prob:proximal-point}) with $q = q_*$ and the optimality condition $\frac{\delta \cL}{\delta q}(q_*) = 0$, we get $\cL(q) - \cL(q_*) \geq \lambda \KL( q \| q_*)$.
This finishes the proof.
\end{proof}
We remark that the inequality (\ref{prob:proximal-point}) indicates that the functional $\cL$ satisfies an analog of strong convexity with the proximal functional $\KL(q'\|q)$, and in the third statement in Proposition \ref{prop:optimization_gap_for_strong_convex_problems}, this convexity plays a similar role as in finite dimensional convex analysis. In particular, inequalities $\lambda \KL( q \| p_q) \geq \cL(q) - \cL(q_*)$ and $\cL(q) - \cL(q_*) \geq \lambda \KL( q \| q_*)$ can be recognized as the counterparts of the Polyak-Łojasiewicz and quadratic growth inequalities, respectively (for details of these conditions see \cite{charles2018stability}). 
Following this analogy, $\KL( q \| p_q)$ and $\KL( q \| q_*)$ act as the squared norm of gradient at $q$ and squared distance between $q$ and $q_*$, respectively. 

Finally, the third statement in Proposition~\ref{prop:optimization_gap_for_strong_convex_problems} indicates that the divergence between $q$ and $p_q$ indeed measures the optimality gap, as expected from the optimality condition (\ref{eq:opt-condition}).
Furthermore, we reveal that $p_q$ can be interpreted as a \textit{proximal point} which minimizes the sum of linearization of $\cL$ and the $\KL$-divergence around $q$, and that convergence of the optimization gap implies convergence to $q_*$ in the sense of $\KL$-divergence. 

\subsection{Convergence Rate in Continuous Time}

We now introduce the convergence rate analysis by utilizing the aforementioned results. 
We first show that the mean field Langevin dynamics (\ref{eq:nonlinear_FP_equation}) converges linearly to the optimal solution of  (\ref{prob:org}) in continuous time under the log-Sobolev inequality. 
\begin{theorem}\label{theorem:convergence}
Let $\{q_t\}_{t \geq 0}$ be the evolution described by (\ref{eq:nonlinear_FP_equation}).
Under Assumption \ref{assumption:potential}, \ref{assumption:ls_ineq}, we get for $t \geq 0$,
\[ \cL(q_t) - \cL(q_*) \leq \exp( -2 \alpha \lambda t)( \cL(q_0) - \cL(q_*) ). \]
\end{theorem}
\begin{proof}
From Proposition \ref{prop:optimization_gap_for_strong_convex_problems} and  (\ref{eq:nonlinear_FP_equation_2}), we have
\begin{align*}
    \frac{\rd }{\rd t}&( \cL(q_t) - \cL(q_*) ) \\
    &= \int \frac{\delta \cL}{\delta q}(q_t)(\theta)\frac{\partial q_t}{\partial t}(\theta) \rd \theta \\
    &= \lambda \int \frac{\delta \cL}{\delta q}(q_t)(\theta) \nabla \cdot \left( q_t(\theta) \nabla \log \frac{q_t}{p_{q_t}}(\theta)\right) \rd \theta \\
    &= - \lambda\int q_t(\theta) \nabla \frac{\delta \cL}{\delta q}(q_t)(\theta)^\top \nabla \log \frac{q_t}{p_{q_t}}(\theta) \rd \theta \\
    &= - \lambda^2\int q_t(\theta) \left\| \nabla \log \frac{q_t}{p_{q_t}}(\theta) \right\|_2^2 \rd \theta \\
    &\leq - 2\alpha \lambda^2 \KL( q_t \| p_{q_t}) \\
    &\leq - 2\alpha\lambda ( \cL(q_t) - \cL(q_*) ).
\end{align*}
The statement then follows from a straightforward application of the Grönwall’s inequality. 
\end{proof}

As a corollary, we can also show the convergence of $\KL(q_t \| p_{q_t})$ in the following sense. 
\begin{corollary}\label{corollary:convergence}
Under the same setting as Theorem \ref{theorem:convergence}, we have for $t \geq 1$,
\[ \inf_{s \in [0,t]} \KL( q_{s} \| p_{q_{s}}) \leq \frac{\exp( -2 \alpha \lambda (t-1))}{2\alpha \lambda^2}( \cL(q_0) - \cL(q_*) ). \]
\end{corollary}

\subsection{Convergence Rate in Discrete Time}\label{subsec:discrete-time}

For the standard Langevin dynamics (i.e., linear Fokker-Planck equation), exponential convergence in continuous time often implies the same convergence up to certain error (depending on the step size) in discrete time. 
In this section we show that the same property also holds for the mean field Langevin dynamics (\ref{eq:MF-Langevin-Dynamics}). 
We provide a convergence rate analysis of the discrete-time dynamics (\ref{eq:noisy-GD}) by adapting a version of the ``one-step interpolation'' argument presented in \cite{vempala2019rapid}. 
Since analysis of one single step of the dynamics (\ref{eq:noisy-GD}) is the key, we adopt the following notations for conciseness.
Let $\theta^{(k)} \sim q^{(k)}(\theta)\rd\theta$ be a random variable that represents the current iteration, and let $\theta^{(k+1)}_t$ be the next iterate of noisy gradient descent with step size $t>0$: 
\begin{equation}\label{eq:noisy-GD-2}
    \theta^{(k+1)}_t = \theta^{(k)} - t \nabla \frac{\delta F}{\delta q}(q^{(k)})(\theta^{(k)}) + \sqrt{2\lambda t}\xi^{(k)},     
\end{equation}
where $\xi^{(k)} \sim \cN(0,I_d)$.
We denote a probability distribution of $\theta^{(k+1)}_t$ by $q^{(k+1)}_t(\theta)\rd \theta$.
Note that this step is equivalent to (\ref{eq:noisy-GD}) when $t = \eta$, that is, $\theta^{(k+1)}_\eta = \theta^{(k+1)}$ and $q^{(k+1)}_\eta = q^{(k+1)}$.
We define $\delta_{q^{(k)},t}$ as
\begin{equation*}
\bE_{(\theta^{(k)},\theta^{(k+1)}_t)} \!\left\| \nabla  \frac{\delta F}{\delta q}(q^{(k)})(\theta^{(k)}) \!-\! \nabla \frac{\delta F}{\delta q}(q^{(k+1)}_t)(\theta^{(k+1)}_t)  \right\|_2^2,
\end{equation*}
where $\bE_{(\theta^{(k)},\theta^{(k+1)}_t)}$ is an expectation in the joint distribution of $\theta^{(k)}$ and $\theta^{(k+1)}_t$. 
Note that the term $\delta_{q^{(k)},t}$ can be recognized as a discretization error which comes from the positive step size $t > 0$, and this error usually decays to zero in common settings as $t \to 0$. 
Thus, by carefully incorporating this error into the proof of Theorem \ref{theorem:convergence}, we can show the convergence of the discrete-time (\ref{eq:noisy-GD}) up to a certain error in proportion to a constant $\delta_\eta$ depending on $\eta$. The complete proof is deferred to Appendix \ref{subsec:proof_discrete-time}.

\begin{theorem}\label{theorem:discrete-time}
Let $\{ \theta^{(k)}\}_{k=0}^\infty$ be the iterations of random variables generated by the discrete-time dynamics (\ref{eq:noisy-GD}) with the step size $\eta$ and $\{q^{(k)}\}_{k=0}^\infty$ be the corresponding probability distributions.
Suppose Assumption \ref{assumption:potential}, \ref{assumption:ls_ineq} hold and there exists a constant $\delta_\eta$ such that $\delta_{q^{(k)},t} \leq \delta_\eta$ for any $0 <  t \leq \eta$ and non-negative integer $k$.
Then, it follows that 
\begin{align*}
\cL&(q^{(k)}) - \cL(q_*)  
\leq 
\\ &\frac{ \delta_{\eta}}{2 \alpha \lambda} 
+ \exp( -\alpha\lambda\eta k )\left( \cL(q^{(0)}) - \cL(q_*) \right). 
\end{align*}
\end{theorem}
\begin{proof}[Proof sketch]
We first present the one-step analysis for iteration ($\ref{eq:noisy-GD-2}$) with step size $\eta$.
Consider the stochastic differential equation:
\begin{equation}\label{eq:noisy-GD_dynamics_copy}
    \rd\theta_t = - \nabla \frac{\delta F}{\delta q}(q^{(k)})(\theta_0)\rd t + \sqrt{2\lambda}\rd W_t,
\end{equation}
where $\theta_0 = \theta^{(k)}$ and $W_t$ is the Brownian motion in $\bR^d$ with $W_0 = 0$.
Then, (\ref{eq:noisy-GD-2}) is the solution of this equation at time $t$.
We denote by $q_{0t}(\theta_0,\theta_t)$ the joint probability distribution of $(\theta_0,\theta_t)$ for time $t$, 
and by $q_{t|0},~q_{0|t}$ and $q_0,~q_t$ conditional and marginal distributions. That is, $q_0 = q^{(k)}$, $q_t = q^{(k+1)}_t$ (i.e., $\theta_t \disteq \theta^{(k+1)}_t$), and
\[ q_{0t}(\theta_0,\theta_t) = q_0(\theta_0) q_{t|0}(\theta_t | \theta_0) = q_t(\theta_t) q_{0|t}(\theta_0 | \theta_t). \]
The continuity equation of $q_{t|0}$ conditioning on $\theta_0$ can be described as follows (see Section 7 of \cite{vempala2019rapid} for details):
\begin{align*}
    \frac{\partial q_{t|0}(\theta_t|\theta_0) }{\partial t} 
    =& \nabla \cdot \left( q_{t|0}(\theta_t|\theta_0) \nabla \frac{\delta F}{\delta q}(q_0)(\theta_0) \right) \\
    &+ \lambda \Delta q_{t|0}(\theta_t|\theta_0).
\end{align*}
Therefore, we obtain the following description of $q_t$:
\begin{align}
    \frac{\partial q_{t}(\theta_t) }{\partial t}
    &=\lambda \nabla \cdot \left( q_t(\theta_t) \nabla \log\frac{q_t}{p_{q_t}}(\theta_t) \right) \notag\\ 
    &+\nabla \cdot \biggl\{ q_t(\theta_t) \biggl( \bE_{\theta_0|\theta_t}\left[ \nabla \frac{\delta F}{\delta q}(q_0)(\theta_0) \middle| \theta_t \right] \notag \\
    &- \nabla \frac{\delta F}{\delta q}(q_t)(\theta_t) \biggr) \biggr\}, \label{eq:theorem:error-FP_copy}
\end{align}
where $p_{q_t}(\cdot) \propto \exp\left( -\frac{1}{\lambda} \nabla \frac{\delta F}{\delta q}(q_t)(\cdot) \right)$. 
By Assumption \ref{assumption:ls_ineq} and (\ref{eq:theorem:error-FP_copy}), for $0 \leq t \leq \eta$, we have
\begin{align*}
    \frac{\rd \cL}{\rd t}(q_t) 
    &= \int \frac{\delta \cL}{\delta q}(q_t)(\theta)\frac{\partial q_t}{\partial t}(\theta) \rd \theta \\
    &\leq - \frac{\lambda^2}{2} \int q_t(\theta) \left\| \nabla \log \frac{q_t}{p_{q_t}}(\theta) \right\|_2^2 \rd \theta \\
    &~~ + \frac{1}{2} \bE_{(\theta_0,\theta)\sim q_{0t}} \! \left\| \nabla \frac{\delta F}{\delta q}(q_0)(\theta_0) 
    \! - \! \nabla \frac{\delta F}{\delta q}(q_t)(\theta) \right\|_2^2 \\
    &\leq - \alpha\lambda ( \cL(q_t) - \cL(q_*) ) + \frac{1}{2} \delta_{\eta},
\end{align*}
where we used $(\theta_0,\theta_t) \disteq (\theta^{(k)},\theta^{(k+1)}_t)$ to bound the last expectation by $\delta_{q^{(k)},t} \leq \delta_\eta$. 
Noting $q_\eta = q^{(k+1)}$ and $q_0 = q^{(k)}$, Grönwall's inequality yields
\begin{align*}
\cL &(q^{(k+1)}) - \cL(q_*) - \frac{ \delta_{\eta}}{2 \alpha \lambda} \\
&\leq \exp( -\alpha\lambda\eta )\left( \cL(q^{(k)}) - \cL(q_*) - \frac{ \delta_{\eta}}{2 \alpha \lambda} \right). 
\end{align*}

This reduction holds at every iteration of (\ref{eq:noisy-GD_dynamics_copy}), which concludes the proof. 
\end{proof}

The proof indicates that the continuous equation (\ref{eq:theorem:error-FP_copy}) of the dynamics (\ref{eq:noisy-GD_dynamics_copy}) associated with noisy gradient descent contains a discretization error term compared to the continuous-time counterpart (\ref{eq:MF-Langevin-Dynamics}).
Usually, this error depends on the step size $\eta$, hence $\eta$ should be chosen sufficiently small to achieve a required optimization accuracy.
Moreover, to derive a convergence rate, the relationship between $\eta$ and the discretization error should be explicitly characterized. 
The following lemma outlines this dependency in the case of mean field neural networks~\eqref{eq:mean-field-nn}.   

\begin{lemma}\label{lemma:F-diff_second_moment}
Consider the loss minimization setting using mean field neural networks in Section \ref{subsec:mfnn} with $r(\theta)=\|\theta\|_2^2$.
Suppose that $\ell(\cdot,y)$ and $h(\cdot,x)$ are differentiable and sufficiently smooth, that is, there exist positive constants $C_1,\ldots,C_4$ such that 
$| \partial_z \ell (z,y)| \leq C_1$, 
$| \partial_z\ell(z,y)- \partial_z\ell(z',y) |\leq C_2 |z-z'|$,
$\| \partial_\theta h_{\theta}(x)\|_2 \leq C_3$, and 
$\| \partial_\theta h_{\theta}(x) - \partial_\theta h_{\theta'}(x) \|_2 \leq C_4\|\theta - \theta'\|_2$.
Also suppose $2\lambda' \eta < 1$ and the following condition holds for the iterate $\theta^{(k)}\sim q^{(k)}(\theta)\rd\theta$.
\begin{equation}\label{lem:F-diff_second_moment}
    \bE_{\theta^{(k)}} [ \|\theta^{(k)}\|_2^2] \leq \frac{\eta C_1^2 C_3^2 + 2\lambda d}{2\eta \lambda'^2}.     
\end{equation} 
Then, we get for any $0 < t \leq \eta$,
\begin{equation}\label{lem:F-diff_second_main}
\delta_{q^{(k)},t} 
 \leq 40 \eta ( C_2^2 C_3^4 + (C_1C_4 + 2\lambda' )^2 ) ( \eta C_1^2 C_3^2 + \lambda d ).
\end{equation}
In addition, the same bound as (\ref{lem:F-diff_second_moment}) also holds for the second moment of the next iterate $\|\theta^{(k+1)}\|_2$.
\end{lemma}
This is to say, $\delta_{q^{(k)},t} = O(\eta)$ for all $k$ as long as (\ref{lem:F-diff_second_moment}) holds for $k=0$.
Therefore, in combination with Theorem \ref{theorem:discrete-time}, we arrive at a convergence rate guarantee for the discrete-time dynamics in optimizing mean field neural networks (up to certain error).
Specifically, the following Corollary implies an iteration complexity of $O\left( \frac{1}{\epsilon \alpha^2 \lambda^2} \log \frac{1}{\epsilon}\right)$ to achieve an $\epsilon$-accurate solution. 

\begin{corollary}\label{corollary:discrete-time-complexity}
Consider the same setting as Lemma~\ref{lemma:F-diff_second_moment} and suppose Assumption \ref{assumption:ls_ineq} holds. 
Let $\{ \theta^{(k)}\}_{k=0}^\infty$ be the iterations of random variables generated by the discrete-time dynamics (\ref{eq:noisy-GD}) with the step size $\eta = O(\epsilon \alpha \lambda)$ and $\{q^{(k)}\}_{k=0}^\infty$ be the corresponding probability distributions.
Then if the condition (\ref{lem:F-diff_second_moment}) is satisfied for the initial iterate $\theta^{(0)}\sim q^{(0)}(\theta)\rd\theta$, we know that for any step size $2\lambda' \eta < 1$, the following statement holds true for $k=0,1,2,3...$, 
\begin{align}
&\cL(q^{(k)})- \cL(q_*)=  \notag\\
&O(\epsilon) 
+ \exp\left( - O(\epsilon\alpha^2\lambda^2 k) \right)\left( \cL(q^{(0)}) - \cL(q_*) \right).
\label{eq:linear-rate-discrete}
\end{align}
\end{corollary}

Finally, we remark that discretization error induced by finite-particle approximation can also be controlled via a direct application of Theorem 3 of \cite{mei2018mean}. However, such finite-particle error grows exponentially with the time horizon, and thus is not negligible unless the exponent in the linear convergence \eqref{eq:linear-rate-discrete} is sufficiently large. 
In future work, we intend to investigate conditions under which such exponential dependence can be avoided (e.g., as in \citet{chen2020dynamical}).  

\section{PRIMAL-DUAL VIEWPOINT}\label{sec:primal-dual_viewpoint}
As seen in Section \ref{sec:convergence_analysis}, the proximal Gibbs distribution $p_q$ plays an important role in our convergence rate analysis.
In this section, we complement the previous results by presenting a primal and dual perspective of this proximal distribution in the (regularized) empirical risk minimization setting. Based on this connection, we show that the duality gap can be minimized by the mean field Langevin dynamics (\ref{eq:MF-Langevin-Dynamics}). 
\subsection{Primal-dual Problem}
We first introduce a primal-dual formulation of the empirical risk minimization problem.
For a train dataset $\{(x_i,y_i)\}_{i=1}^n$ and differentiable convex loss function: $\ell(z,y)$ in $z$, we consider the minimization of the following regularized empirical risk:
\begin{equation}\label{prob:erm-primal}
    \cL(q) = \frac{1}{n}\sum_{i=1}^n \ell(h_q(x_i),y_i) 
    + \lambda' \bE_{\theta \sim q}[ \|\theta\|_2^2 ] 
    + \lambda \bE_{q} [ \log q ].
\end{equation}
Note that this problem is a special case of (\ref{prob:org}) by setting $F(q)=\frac{1}{n}\sum_{i=1}^n \ell(h_q(x_i),y_i) + \lambda' \bE_{\theta \sim q}[ \|\theta\|_2^2 ]$.
Write $\ell_i(z) = \ell(z,y_i)$ and $\ell_i^*(\cdot)$ as its Fenchel conjugate, i.e.,
\[ \ell_i^*(z^*) = \sup_{z \in \bR } \{ zz^* - \ell_i(z) \}. ~~\text{for } z^* \in \bR\]
Also, for any given vector $g=\{g_i\}_{i=1}^n \in \bR^n$, we define
\begin{equation*}
    q_g(\theta) 
    = \exp\left( - \frac{1}{\lambda} \left( \frac{1}{n}\sum_{i=1}^n h_{\theta}(x_i)g_i 
    + \lambda'\|\theta\|_2^2 \right) \right).
\end{equation*}
Then, the dual problem of (\ref{prob:erm-primal}) is defined as 
\begin{equation}\label{prob:erm-dual}
    \max_{g \in \bR^n}\left\{ 
    \cD(g) = -\frac{1}{n}\sum_{i=1}^n \ell_i^*(g_i) 
    - \lambda \log\int q_g(\theta) \rd\theta
    \right\}.
\end{equation}
The duality theorem \citep{rockafellar,bauschke2011convex,oko2022psdca} guarantees the relationship $\cD(g) \leq \cL(q)$ for any $g \in \bR^n$ and $q\in \cP$ , and it is known that the duality gap $\cL(q) - \cD(g)$ vanishes at the solutions of (\ref{prob:erm-primal}) and (\ref{prob:erm-dual}) when they exist.
In our problem setting, it is possible to establish a stronger and more precise result.
 
We denote $g_q = \{ \partial_z \ell(z,y_i)|_{z=h_q(x_i)}\}_{i=1}^n \in \bR^n$ ($q \in \cP$). 
The following theorem exactly characterizes the duality gap $\cL(q) - \cD( g_q )$ between $q \in \cP$ and $g_q \in \bR^n$ via the proximal Gibbs distribution $p_q$. 

\begin{theorem}[Duality Theorem]\label{theorem:dual}
Suppose $\ell(\cdot,y)$ is convex and differentiable.
For any $q \in \cP$ the duality gap between $q \in \cP$ and $g_q$ of the problems (\ref{prob:erm-primal}) and (\ref{prob:erm-dual}) is 
\[ 0 \leq \cL(q) - \cD(g_q) = \lambda \KL( q \| p_{q}). \]
\end{theorem}

We make the following observations.
First, this theorem can be seen as a refinement of Proposition \ref{prop:optimization_gap_for_strong_convex_problems}, in which $\lambda \KL (q \| p_{q})$ upper bounds the optimality gap $\cL(q) - \cL(q_*)$ in the general setting. Theorem \ref{theorem:dual} reveals that $\lambda \KL (q \| p_{q})$ is \textit{exactly} the duality gap $\cL(q) - \cD(g_q)$ for the empirical risk minimization problems. 
Second, because of the relationship $p_q(\theta) = \frac{q_{g_q}(\theta)}{\int q_{g_q}(\theta) \rd\theta}$, we notice that the proximal Gibbs distribution $p_q$ of $q$ can be seen as a ``round trip'' between the primal and dual spaces: $q \to g_q \to q_{g_q} \propto p_q$; the equation $\cL(q) - \cD(g_q) = \lambda \KL( q \| p_{q})$ gives an interesting relationship among these variables.
Third, Combining Theorem \ref{theorem:dual} and the duality theorem (e.g., see Proposition 1 of \citet{oko2022psdca}): $\cD(g) \leq \cL(q)$ for any $q \in \cP$ and $g \in \bR^n$, we see that $g_{q_*}$ is the solution of the dual problem (\ref{prob:erm-dual}). 

Another interesting quantity to investigate is the primal objective $\cL(p_q)$ of the proximal Gibbs distribution $p_q$.
The next theorem gives an upper bound on a duality gap between $p_q$ and a dual variable $g_q$.
\begin{theorem}[Second Duality Theorem] \label{theorem:dual2}
Suppose $\ell(\cdot,y)$ is convex and $C_2$-smooth, that is, $|\partial_z\ell(z,y) - \partial_z\ell(z',y)| \leq C_2|z-z'|$, and  $|h_{\theta}(x)| \leq B$.
Then for any $q \in \cP$, the duality gap between $p_q \in \cP$ and $g_q$ of the problems (\ref{prob:erm-primal}) and (\ref{prob:erm-dual}) satisfies
\[ 0 \leq \cL(p_q) - \cD(g_q) \leq (\lambda + 2B^2 C_2) \KL( q \| p_q). \]
\end{theorem}
This theorem provides a new choice for solving the problem (\ref{prob:erm-primal}). That is, after obtaining $q$ by the optimization, we can alternatively utilize $p_q$ as an alternative solution, which may be efficiently approximated by sampling methods for Gibbs distributions.

\subsection{Convergence of Duality Gap}
Combining Theorem \ref{theorem:dual} with Corollary \ref{corollary:convergence}, we see that mean field Langevin dynamics solves both the primal and dual problems in the following sense.
\begin{corollary}\label{convergence:dual}
Consider the evolution $\{q_t\}_{t\geq 0}$ ,which satisfies (\ref{eq:nonlinear_FP_equation}) for the problem (\ref{prob:erm-primal}) under the same settings as in Corollary \ref{corollary:convergence} and Theorem \ref{theorem:dual}.
Let $\{g_{q_t}\}_{t\geq 0}$ be the associated dynamics in $\bR^n$. Then for $t\geq 1$,
\begin{align*}
\inf_{s \in [0,t]} &\{ \cL(q_s) - \cD(g_{g_s}) \}
 \\
&\le\frac{\exp( -2 \alpha \lambda (t-1))}{2\alpha \lambda}\left( \cL(q_0) - \cL(q_*) \right). 
\end{align*}
\end{corollary}
Thus, we can conclude that mean field Langevin dynamics also solves the dual problem, that is,  $\cD(g_{q_t})$ also converges to $\cD(g_{q_*})$.
Following the same reasoning, we can derive a corollary of Theorem \ref{theorem:dual2} demonstrating the convergence of $\cL(p_{q_t}) - \cD(g_{g_t})$. 
  
\paragraph{Evaluation of duality gap.}  
One benefit of this primal-dual formulation is that we can observe the optimization gap by computing the duality gap $\cL(q) - \cD(g)$ or $\cL(p_q) - \cD(g)$, without the knowledge of the optimal value $\cL(q_*)$. 
In Figure~\ref{fig:duality-gap} we empirically demonstrate the duality gap on a regression problem with the squared loss. We set $n=1000, d=5$, and consider a simple student-teacher setting, where the teacher model is a two-layer sigmoid network with orthogonal neurons, and the student model $h_\Theta$ is a two-layer mean field neural network of width $M=1000$ with tanh activation. The student model is optimized by the noisy gradient descent with $\eta=0.01$, and we use the Langevin algorithm to obtain approximate samples from the proximal Gibbs distribution $p_q$. 
For the primal objective $\cL$, we adopt the $k$-nearest neighbors estimator \citep{kozachenko1987sample} with $k=10$ to estimate the entropy; for the dual objective $\cD$, the approximation of the log integral term is described in Appendix~\ref{app:compute-dual}.

\begin{figure}[!htb]   
\vspace{-1mm} 
\centering
\begin{minipage}[t]{0.62\linewidth} 
{\includegraphics[width=1\linewidth]{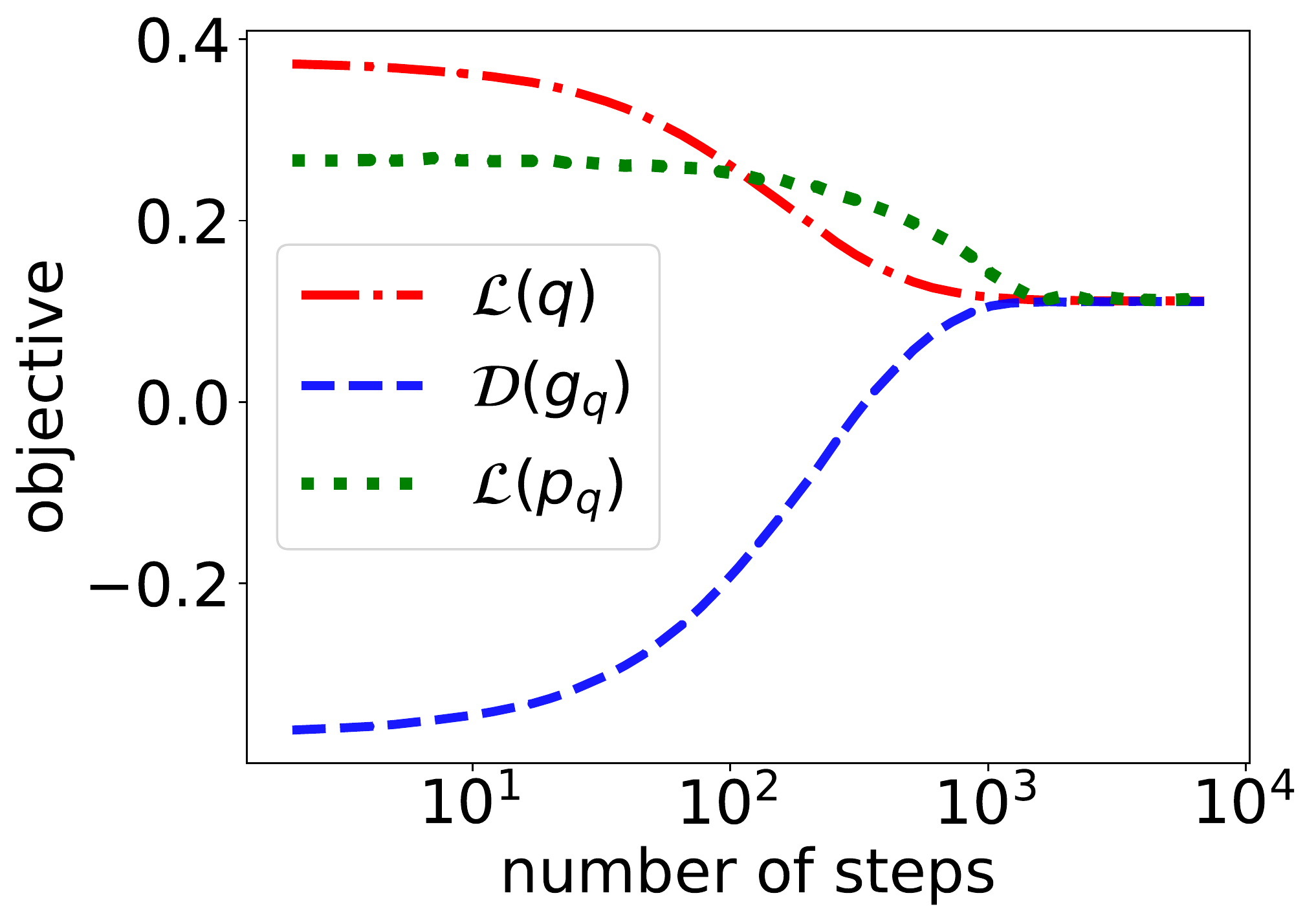}}  
\end{minipage}
\begin{minipage}[b]{0.365\linewidth} 
\captionof{figure}{\small Illustration of duality gap: two-layer tanh network optimizing the empirical squared error. We set $\lambda=\lambda'=10^{-2}$.}    
\label{fig:duality-gap} 
\end{minipage}
\vspace{-5mm}  
\end{figure}    

Observe that towards the end of training, the primal ($\cL(q)$ and $\cL(p_q)$) and dual ($\cD(g_q)$) objectives become close, which is consistent with Theorem~\ref{theorem:dual} and \ref{theorem:dual2}.

\section*{CONCLUSION}
We established quantitative global convergence guarantee for the mean field Langevin dynamics in both continuous- and discrete-time, by adapting convex optimization techniques into the space of measures, in combination with standard analysis of the Langevin dynamics. 
Looking forward, an interesting future direction is to conduct the analysis under weaker isoperimetry conditions such as the Poincar\'{e} inequality, which covers more general objectives. It is also important to refine our convergence result (e.g., exponential dependence on $1/\lambda$ in the LSI constant) under additional structure of the learning problem.   
Another interesting direction is to explore applications of the mean field dynamics beyond the optimization of neural networks.  

\bigskip
\section*{Acknowledgment}
AN was partially supported by JSPS KAKENHI (19K20337) and JST-PRESTO (JPMJPR1928).
DW was partially supported by NSERC and LG Electronics.
TS was partially supported by JSPS KAKENHI (18H03201), Japan Digital Design and JST CREST.

\bigskip
\newpage

\bibliographystyle{apalike}
\bibliography{ref}

\clearpage
\onecolumn
\renewcommand{\thesection}{\Alph{section}}
\renewcommand{\thesubsection}{\Alph{section}. \arabic{subsection}}
\renewcommand{\thetheorem}{\Alph{theorem}}
\renewcommand{\thelemma}{\Alph{lemma}}
\renewcommand{\theproposition}{\Alph{proposition}}
\renewcommand{\thedefinition}{\Alph{definition}}
\renewcommand{\thecorollary}{\Alph{corollary}}
\renewcommand{\theassumption}{\Alph{assumption}}
\renewcommand{\theexample}{\Alph{example}}

\setcounter{section}{0}
\setcounter{subsection}{0}
\setcounter{theorem}{0}
\setcounter{lemma}{0}
\setcounter{proposition}{0}
\setcounter{definition}{0}
\setcounter{corollary}{0}
\setcounter{assumption}{0}

{
\newgeometry{top=1in, bottom=1in,left=1.in,right=1.in}   
\renewcommand{\contentsname}{Table of Contents}
\tableofcontents

\newpage

\allowdisplaybreaks

\linewidth\hsize
{\centering \Large\bfseries Appendix: Convex Analysis of Mean Field Langevin Dynamics \par}

\section{OMITTED PROOFS}

\subsection{Preliminaries: Log-Sobolev Inequality}
\label{app:LSI}

For the proximal Gibbs distribution (\ref{eq:gibbs_dist}) of the mean field neural network, we can verify that the log-Sobolev inequality (Assumption \ref{assumption:ls_ineq}) holds \textit{uniformly} over $q$ under the boundedness assumptions on $\partial_z \ell(z,y)$ and $h_\theta(x)$ and $\ell_2$-regularization $r(\theta) = \|\theta\|_2^2$ by utilizing the following two facts.
Specifically, these two facts indicate that the Gibbs distribution $p_q$ whose potential is the sum of strongly concave function and bounded perturbation satisfies the log-Sobolev inequality. 

It is well-known that strongly log-concave densities satisfy the LSI with a dimension-free constant (up to the spectral norm of the covariance).
\begin{example}[\cite{bakry1985diffusions}]\label{ex:quadratic_sobolev}
Let $q \propto \exp(-f)$ be a probability density, where $f : \bR^p \rightarrow \bR$ is a smooth function. 
If there exists $c > 0$ such that $\nabla^2 f \succeq cI_p$, then $q(\theta)\rd\theta$ satisfies Log-Sobolev inequality with constant $c$.
\end{example}

In addition, the log-Sobolev inequality is preserved under bounded perturbation, as originally shown in \cite{holley1987logarithmic}. 
\begin{lemma}[\cite{holley1987logarithmic}]
\label{lem:sobolev_perturbation}
Let $q(\theta)\rd\theta$ be a probability distribution on $\bR^p$ satisfying the log-Sobolev inequality with a constant $\alpha$.
For a bounded function $B: \bR^p \rightarrow \bR$, we define a probability distribution $q_B(\theta)\rd\theta$ as follows:
\[ q_B(\theta)\rd \theta = \frac{\exp(B(\theta))q(\theta)}{ \bE_q[\exp(B(\theta))]} \rd\theta. \]
Then, $q_B \rd\theta$ satisfies the log-Sobolev inequality with a constant $\alpha / \exp(4\|B\|_\infty)$. 
\end{lemma}
In our case, if $| \partial_z \ell (z,y)| \leq C_1$, $|h_\theta(x)| \leq C_5$, and $r(\theta) = \|\theta\|_2^2$, then the proximal Gibbs distribution (\ref{eq:gibbs_dist}) satisfies the log-Sobolev inequality with a constant $\alpha = \frac{2\lambda'}{\lambda \exp( 4C_1C_5\lambda^{-1} )}$. 
We remark that the exponential dependence in the LSI constant may be unavoidable in the most general setting \citep{menz2014poincare}. 

\bigskip

\subsection{Continuous Time Analysis}
\begin{proof}[Proof of Corollary \ref{corollary:convergence}]
In the same way as the proof of Theorem \ref{theorem:convergence}, 
 \[ 2\alpha \lambda^2 \KL( q_s \| p_{q_s}) \leq - \frac{\rd }{\rd s}( \cL(q_s) - \cL(q_*) ). \]
By taking integral of this inequality on the interval $[t-1,t]$~($t\in \bN$), we get
\begin{align*}
    2\alpha \lambda^2 \int_{t-1}^t\KL( q_s \| p_{q_s}) \rd s 
    &\leq \cL(q_{t-1}) - \cL(q_*) - (\cL(q_{t}) - \cL(q_*))\\
    &\leq \exp( -2 \alpha \lambda (t-1))( \cL(q_0) - \cL(q_*) ).
\end{align*}
Therefore, there exists $s_t \in [t-1,t]$~($t \in \bN$) such that 
\[ \KL( q_{s_t} \| p_{q_{s_t}}) \leq \frac{\exp( -2 \alpha \lambda (t-1))}{2\alpha \lambda^2}( \cL(q_0) - \cL(q_*) ). \]
By taking the infimum of $\KL( q_{s} \| p_{q_{s}})$ over $[0,t]$~($t\in\bN$), we have 
\[ \inf_{s \in [0,t]} \KL( q_{s} \| p_{q_{s}}) \leq \frac{\exp( -2 \alpha \lambda (t-1))}{2\alpha \lambda^2}( \cL(q_0) - \cL(q_*) ). \]
\end{proof}

\subsection{Discrete Time Analysis}\label{subsec:proof_discrete-time}
\begin{proof}[Proof of Theorem \ref{theorem:discrete-time}]
We first present the one step analysis for the iteration ($\ref{eq:noisy-GD-2}$) with the step size $\eta$. 
Let us consider the stochastic differential equation:
\begin{equation}\label{eq:noisy-GD_dynamics}
    \rd\theta_t = - \nabla \frac{\delta F}{\delta q}(q^{(k)})(\theta_0)\rd t + \sqrt{2\lambda}\rd W_t,
\end{equation}
where $\theta_0 = \theta^{(k)}$ and $W_t$ is the Brownian motion in $\bR^d$ with $W_0 = 0$.
Then, the step (\ref{eq:noisy-GD-2}) is the solution of this equation at time $t$.
We denote by $q_{0t}(\theta_0,\theta_t)$ the joint probability distribution of $(\theta_0,\theta_t)$ for time $t$, 
and by $q_{t|0},~q_{0|t}$ and $q_0,~q_t$ conditional and marginal distributions. That is, it holds that $q_0 = q^{(k)}$, $q_t = q^{(k+1)}_t$ (i.e., $\theta_t \disteq \theta^{(k+1)}_t$), and
\[ q_{0t}(\theta_0,\theta_t) = q_0(\theta_0) q_{t|0}(\theta_t | \theta_0) = q_t(\theta_t) q_{0|t}(\theta_0 | \theta_t). \]
The continuity equation of $q_{t|0}$ conditioned on $\theta_0$ is given as (see Section 7 of \cite{vempala2019rapid} for details):
\begin{equation*}
    \frac{\partial q_{t|0}(\theta_t|\theta_0) }{\partial t} 
    = \nabla \cdot \left( q_{t|0}(\theta_t|\theta_0) \nabla \frac{\delta F}{\delta q}(q_0)(\theta_0) \right) + \lambda \Delta q_{t|0}(\theta_t|\theta_0).
\end{equation*}
Therefore, we obtain the continuity equation of $q_t$:
\begin{align}
    \frac{\partial q_{t}(\theta_t) }{\partial t}
    &= \int \frac{\partial q_{t|0}(\theta_t|\theta_0) }{\partial t} q_0(\theta_0) \rd\theta_0 \notag\\
    &= \int \left( \nabla \cdot \left( q_{0t}(\theta_0,\theta_t) \nabla \frac{\delta F}{\delta q}(q_0)(\theta_0) \right) 
    + \lambda \Delta q_{0t}(\theta_0,\theta_t) \right) \rd\theta_0  \notag\\
    &= \nabla \cdot \left( q_t(\theta_t) \int q_{0|t}(\theta_0|\theta_t) \nabla \frac{\delta F}{\delta q}(q_0)(\theta_0) \rd\theta_0 \right)
    + \lambda \Delta q_{t}(\theta_t) \notag\\
    &= \nabla \cdot \left( q_t(\theta_t) \left( \bE_{\theta_0|\theta_t}\left[ \nabla \frac{\delta F}{\delta q}(q_0)(\theta_0) \middle| \theta_t \right]
    + \lambda \nabla \log q_t (\theta_t) \right) \right) \notag\\
    &= \lambda \nabla \cdot \left( q_t(\theta_t) \nabla \log\frac{q_t}{p_{q_t}}(\theta_t) \right) \notag\\ 
    &+\nabla \cdot \left( q_t(\theta_t) \left( \bE_{\theta_0|\theta_t}\left[ \nabla \frac{\delta F}{\delta q}(q_0)(\theta_0) \middle| \theta_t \right] 
    - \nabla \frac{\delta F}{\delta q}(q_t)(\theta_t) \right) \right), \label{eq:theorem:error-FP}
\end{align}
where $p_{q_t}(\cdot) \propto \exp\left( -\frac{1}{\lambda} \nabla \frac{\delta F}{\delta q}(q_t)(\cdot) \right)$.
For simplicity, we write $\delta_t(\cdot) = \bE_{\theta_0 \sim q_{0|t}}\left[ \nabla \frac{\delta F}{\delta q}(q_0)(\theta_0) \middle| \theta_t=\cdot \right] - \nabla \frac{\delta F}{\delta q}(q_t)(\cdot)$.
By Assumption \ref{assumption:ls_ineq} and (\ref{eq:theorem:error-FP}), for $0 \leq t \leq \eta$, we have
\begin{align*}
    \frac{\rd \cL}{\rd t}(q_t) 
    &= \int \frac{\delta \cL}{\delta q}(q_t)(\theta)\frac{\partial q_t}{\partial t}(\theta) \rd \theta \\
    &= \lambda \int \frac{\delta \cL}{\delta q}(q_t)(\theta) \nabla \cdot \left( q_t (\theta)\nabla \log \frac{q_t}{p_{q_t}}(\theta)\right) \rd \theta \\
    &+ \int \frac{\delta \cL}{\delta q}(q_t)(\theta) \nabla \cdot \left( q_t(\theta) \delta_t(\theta)\right) \rd \theta\\
    &= - \lambda\int q_t(\theta) \nabla \frac{\delta \cL}{\delta q}(q_t)(\theta)^\top \nabla \log \frac{q_t}{p_{q_t}}(\theta) \rd \theta \\
    &- \int q_t(\theta) \nabla \frac{\delta \cL}{\delta q}(q_t)(\theta)^\top \delta_t(\theta) \rd \theta\\
    &= - \lambda^2 \int q_t(\theta) \left\| \nabla \log \frac{q_t}{p_{q_t}}(\theta) \right\|_2^2 \rd \theta \\
    &- \int q_{0t}(\theta_0,\theta) \lambda\nabla \log \frac{q_t}{p_{q_t}}(\theta)^\top \left( \nabla \frac{\delta F}{\delta q}(q_0)(\theta_0) 
    - \nabla \frac{\delta F}{\delta q}(q_t)(\theta) \right) \rd \theta_0 \rd\theta\\
    &\leq - \lambda^2 \int q_t(\theta) \left\| \nabla \log \frac{q_t}{p_{q_t}}(\theta) \right\|_2^2 \rd \theta \\
    &+  \frac{1}{2} \int q_{0t}(\theta_0,\theta) \left( \lambda^2\left\| \nabla \log \frac{q_t}{p_{q_t}}(\theta) \right\|_2^2 
    + \left\| \nabla \frac{\delta F}{\delta q}(q_0)(\theta_0) 
    - \nabla \frac{\delta F}{\delta q}(q_t)(\theta) \right\|_2^2 \right) \rd\theta_0 \rd\theta\\
    &\leq - \frac{\lambda^2}{2} \int q_t(\theta) \left\| \nabla \log \frac{q_t}{p_{q_t}}(\theta) \right\|_2^2 \rd \theta \\
    &+ \frac{1}{2} \bE_{(\theta_0,\theta)\sim q_{0t}}\left[ \left\| \nabla \frac{\delta F}{\delta q}(q_0)(\theta_0) 
    - \nabla \frac{\delta F}{\delta q}(q_t)(\theta) \right\|_2^2 \right] \\
    &\leq - \alpha \lambda^2 \KL( q_t \| p_{q_t}) + \frac{1}{2} \delta_{q_0,t} \\
    &\leq - \alpha\lambda ( \cL(q_t) - \cL(q_*) ) + \frac{1}{2} \delta_{\eta},
\end{align*}
where we used $(\theta_0,\theta_t) \disteq (\theta^{(k)},\theta^{(k+1)}_t)$ to bound the last expectation by $\delta_{q^{(k)},t} \leq \delta_\eta$. 
Thus, for $0 \leq t \leq \eta$,  we get 
\[ \frac{\rd }{\rd t}\left( \cL(q_t) - \cL(q_*) - \frac{ \delta_{\eta}}{2 \alpha \lambda} \right) 
\leq - \alpha\lambda \left( \cL(q_t) - \cL(q_*) -  \frac{\delta_\eta}{2\alpha\lambda} \right). \]
Noting $q_\eta = q^{(k+1)}$ and $q_0 = q^{(k)}$, the Gronwall's inequality leads to 
\[ \cL(q^{(k+1)}) - \cL(q_*) - \frac{ \delta_{\eta}}{2 \alpha \lambda} 
\leq \exp( -\alpha\lambda\eta )\left( \cL(q^{(k)}) - \cL(q_*) - \frac{ \delta_{\eta}}{2 \alpha \lambda} \right). \]

This reduction holds at every iteration of (\ref{eq:noisy-GD_dynamics}). Hence, we arrive at the desired result,  
\[ \cL(q^{(k)}) - \cL(q_*)  
\leq \frac{ \delta_{\eta}}{2 \alpha \lambda} 
+ \exp( -\alpha\lambda\eta k )\left( \cL(q^{(0)}) - \cL(q_*) - \frac{ \delta_{\eta}}{2 \alpha \lambda} \right). \]
\end{proof}

\begin{proof}[Proof of Lemma \ref{lemma:F-diff_second_moment}]
For notational simplicity, we use $\theta, \theta', q, q'$ to represent $\theta^{(k)}, \theta^{(k+1)}_t, q^{(k)}, q^{(k+1)}_t$ appearing in $\delta_{q^{(k)},t}$.
Recall the definition of $F$ in Section \ref{subsec:mfnn}, we have
\begin{align}
     \left\| \nabla  \frac{\delta F}{\delta q}(q)(\theta) -  \nabla \frac{\delta F}{\delta q}(q')(\theta')  \right\|_2 
     &\leq \bE_{(X,Y)}\left[\left\| \partial_z\ell(h_q(X),Y) \partial_\theta h_{\theta}(X) 
     - \partial_z\ell(h_{q'}(X),Y) \partial_\theta h_{\theta'}(X)  \right\|_2 \right] \notag\\
     &+ 2 \lambda' \| \theta -  \theta' \|_2 \notag\\
     &\leq \bE_{(X,Y)}\left[ \left\| (\partial_z\ell(h_q(X),Y)- \partial_z\ell(h_{q'}(X),Y)) \partial_\theta h_{\theta}(X)  \right\|_2 \right] \notag\\
     &+ \bE_{(X,Y)}\left[ \left\| \partial_z\ell(h_{q'}(X),Y)( \partial_\theta h_{\theta}(X) - \partial_\theta h_{\theta'}(X) )  \right\|_2 \right] \notag\\
     &+ 2 \lambda' \| \theta -  \theta' \|_2 \notag\\
     &\leq C_2 C_3 \bE_{X}[ | h_q(X) - h_{q'}(X) | ]
     + (C_1C_4 + 2\lambda' )\left\| \theta - \theta' \right\|_2. \label{lem:F_diff_1}
\end{align}
The expectation of $\left\| \theta - \theta' \right\|_2^2$ can be bounded as follows:
\begin{align}
    \bE_{(\theta,\theta')}\left[ \left\| \theta - \theta' \right\|_2^2 \right]
    &= \bE_{(\theta,\xi)}\left[ \left\| t \nabla \frac{\delta F}{\delta q}(q)(\theta) - \sqrt{2\lambda t}\xi \right\|_2^2 \right]  \notag\\
    &\leq 2t^2 \bE_{\theta}\left[ \left\| \nabla \frac{\delta F}{\delta q}(q)(\theta) \right\|_2^2 \right]
    + 4\lambda t \bE_{\xi}\left[ \left\|\xi \right\|_2^2 \right]  \notag\\
    &\leq 4t^2 \bE_{\theta}\left[ \left\| \bE_{(X,Y)}\left[  \partial_z\ell(h_q(X),Y) \partial_\theta h_{\theta}(X) \right] \right\|_2^2 
    + 4\lambda'^2\| \theta \|_2^2 \right]
    + 4\lambda t d \notag\\
    &\leq 4t^2 C_1^2 C_3^2
    + 16t^2 \lambda'^2 \bE_{\theta}\left[ \| \theta \|_2^2 \right]
    + 4\lambda t d \notag\\
    &\leq 20\eta ( \eta C_1^2 C_3^2 + \lambda d ). \label{lem:F_diff_2}
\end{align}

Moreover, $| h_q(x) - h_{q'}(x) |$ can be bounded as follows:
\begin{align}
    | h_q(x) - h_{q'}(x) | 
    &= \left|\bE_{(\theta,\theta')}[ h_{\theta}(x)  - h_{\theta'}(x) ] \right| \notag \\
    &\leq C_3 \bE_{(\theta,\theta')}\left[\| \theta - \theta' \|_2 \right] \notag \\
    &\leq C_3 \sqrt{ \bE_{(\theta,\theta')}\left[\| \theta - \theta' \|_2^2 \right] }. \label{lem:F_diff_3}
\end{align}

Therefore, we get by (\ref{lem:F_diff_1}), (\ref{lem:F_diff_2}), and (\ref{lem:F_diff_3}),
\begin{align*}
     \bE_{(\theta,\theta')}\left[ \left\| \nabla  \frac{\delta F}{\delta q}(q)(\theta) -  \nabla \frac{\delta F}{\delta q}(q')(\theta')  \right\|_2^2 \right]
     &\leq 2 C_2^2 C_3^2 (\bE_{X}[ | h_q(X) - h_{q'}(X) | ] )^2 \\
     &+ 2(C_1C_4 + 2\lambda' )^2 \bE_{(\theta,\theta')}\left[  \left\| \theta - \theta' \right\|_2^2 \right] \\
     &\leq 2 ( C_2^2 C_3^4 + (C_1C_4 + 2\lambda' )^2 ) \bE_{(\theta,\theta')}\left[\| \theta - \theta' \|_2^2 \right] \\
     &\leq 40 \eta ( C_2^2 C_3^4 + (C_1C_4 + 2\lambda' )^2 ) ( \eta C_1^2 C_3^2 + \lambda d ).
\end{align*}

Finally, we show the same bound on second moment of $\|\theta^{(k+1)}\|_2$ as on $\|\theta^{(k)}\|_2$.
Using the inequality $(a+b)^2 \leq \left( 1 + \gamma \right)a^2 + \left(1+ \frac{1}{\gamma}\right)b^2$ with $\gamma = \frac{2\eta\lambda'}{1-2\eta \lambda'}$, we get
\begin{align*}
    \bE_{\theta^{(k+1)}}\left[ \| \theta^{(k+1)}\|_2^2 \right]
    &= \bE_{(\theta,\xi)}\left[ \left\| (1-2\lambda' \eta) \theta - \eta \bE_{(X,Y)}\left[  \partial_z\ell(h_q(X),Y) \partial_\theta h_{\theta}(X) \right] + \sqrt{2\lambda \eta}\xi \right\|_2^2 \right] \\
    &\leq \bE_{(\theta,\xi)}\left[ \left( (1-2\lambda' \eta) \|\theta\|_2 + \eta C_1 C_3 + \sqrt{2\lambda \eta} \left\| \xi \right\|_2 \right)^2 \right] \\
    &= \bE_{(\theta,\xi)}\left[ \left(1+\gamma \right) (1-2\lambda' \eta)^2 \|\theta\|_2^2 + \left(1+\frac{1}{\gamma}\right)\left( \eta C_1 C_3 + \sqrt{2\lambda \eta} \left\| \xi \right\|_2 \right)^2 \right] \\
    &\leq (1-2\lambda' \eta) \bE_{\theta}\left[ \|\theta\|_2^2 \right] 
    + \frac{1}{\lambda'}\left( \eta C_1^2 C_3^2 + 2\lambda \bE_{\xi}\left[ \left\| \xi \right\|_2^2 \right] \right) \\
    &\leq (1-2\lambda' \eta) \frac{\eta C_1^2 C_3^2 + 2\lambda d}{2\eta\lambda'^2}
    + \frac{1}{\lambda'}\left( \eta C_1^2 C_3^2 + 2\lambda d \right) \\
    &= \frac{\eta C_1^2 C_3^2 + 2\lambda d}{2\eta \lambda'^2}.
\end{align*}
\end{proof}

\subsection{Duality Theorems}
\begin{proof}[Proof of Theorem \ref{theorem:dual}]
It is well known that $\nabla \ell_i$ and $\nabla \ell_i^*$ give inverse mapping to each other. 
Hence, we know that $\nabla \ell_i^* (g_{q,i}) = \nabla \ell_i^* ( \nabla \ell_i (h_q(x_i) ) ) = h_q(x_i)$ and 
\begin{align}
    \ell_i^*( g_{q,i}) 
    = \sup_{z \in \bR } \{ z g_{q,i} - \ell_i(z) \} 
    = h_q(x_i) g_{q,i} - \ell_i ( h_q(x_i) ). \label{eq:convex_conjugate_eq}
\end{align}
Recall the definitions of $p_q$, $g_q$, and $q_{g}$, we see that $q_{g_q} \propto p_q$, and hence
\begin{align}
    \KL( q \| p_{q}) 
    = \bE_{\theta \sim q}\left[ \log \frac{q}{p_{q}}(\theta) \right] 
    = \bE_{\theta \sim q}\left[ \log q(\theta) - \log q_{g_q}(\theta) \right]
    + \log \int q_{g_q}(\theta) d\theta. \label{eq:kl_eq}
\end{align}
Combining (\ref{eq:convex_conjugate_eq}) and (\ref{eq:kl_eq}), we get for any $q \in \cP$,
\begin{align*}
    \cD(g_q) 
    &= -\frac{1}{n}\sum_{i=1}^n \ell_i^*( g_{q,i} ) 
    - \lambda \log\int q_{g_q}(\theta) \rd\theta \\
    &= \frac{1}{n} \sum_{i=1}^n \left( \ell_i ( h_q(x_i) ) - h_q(x_i) g_{q,i} \right)
    - \lambda \KL( q \| p_{q}) + \lambda \bE_{\theta \sim q}\left[ \log q(\theta) - \log q_{g_q}(\theta) \right] \\
    &= \frac{1}{n} \sum_{i=1}^n \left( \ell_i ( h_q(x_i) ) - h_q(x_i) g_{q,i} \right)
    - \lambda \KL( q \| p_{q}) + \lambda \bE_{\theta \sim q}\left[ \log q(\theta) \right] 
    + \bE_{\theta \sim q}\left[ \frac{1}{n}\sum_{i=1}^n h_{\theta}(x_i)g_{q,i} + \lambda' \|\theta\|_2^2 \right]\\
    &= \cL(q) - \lambda \KL( q \| p_{q}).
\end{align*}
This concludes the proof.
\end{proof}

\begin{proof}[Proof of Theorem \ref{theorem:dual2}]
From direct computation, we get for $q \in \cP$,
\begin{align*}
    \frac{\delta \cL}{\delta q'}(q')|_{q'=p_q}(\theta)
    &= \frac{1}{n}\sum_{i=1}^n \partial_z\ell(h_{p_q}(x_i),y_i)h_{\theta}(x_i) + \lambda' \|\theta\|_2^2 
    + \lambda \log p_q (\theta) \\
    &= \frac{1}{n}\sum_{i=1}^n (\partial_z\ell(h_{p_q}(x_i),y_i) - \partial_z\ell(h_{q}(x_i),y_i))h_{\theta}(x_i)  
    + const.
\end{align*}
Hence, we have
\begin{align*}
    \left| \int  \frac{\delta \cL}{\delta q'}(q')|_{q'=p_q}(\theta) (q-p_q)(\theta) \rd\theta \right|
    &= \left| \int \frac{1}{n}\sum_{i=1}^n (\partial_z\ell(h_{p_q}(x_i),y_i) - \partial_z\ell(h_{q}(x_i),y_i))h_{\theta}(x_i) (q-p_q)(\theta) \rd\theta \right|\\
    &= \left| \frac{1}{n}\sum_{i=1}^n (\partial_z\ell(h_{p_q}(x_i),y_i) - \partial_z\ell(h_{q}(x_i),y_i)) ( h_q(x_i) - h_{p_q}(x_i) ) \right|\\
    &\leq \frac{C_2}{n}\sum_{i=1}^n | h_{q}(x_i) - h_{p_q}(x_i)|^2 \\
    &\leq B^2 C_2\|q - p_q\|_{L_1(\rd\theta)}^2 \\
    &\leq 2B^2 C_2 \KL(q\|p_q),
\end{align*}
where we used Pinsker's inequality for the last inequality. 

By the convexity of $\cL$, we get
\begin{align*}
    \cL(p_q) - \cD(g_q) - \lambda \KL(q\|p_q)
    &= \cL(p_q) - \cL(q) \\
    &\leq - \int  \frac{\delta \cL}{\delta q'}(q')|_{q'=p_q}(\theta) (q-p_q)(\theta) \rd\theta \\
    &\leq 2B^2 C_2 \KL(q\|p_q).
\end{align*}

Therefore, we obtain $\cL(p_q) - \cD(g_q) \leq (\lambda + 2B^2 C_2) \KL(q\|p_q)$.
\end{proof}

\bigskip

\section{ADDITIONAL DETAILS}

\subsection{Computation of the Dual Objective}
\label{app:compute-dual}
We briefly outline the estimation of the dual objective (\ref{prob:erm-dual}), which consists of the sum of Fenchel conjugate functions and the normalization term as follows:
\begin{equation*}
    \cD(g) = -\frac{1}{n}\sum_{i=1}^n \ell_i^*(g_i) 
    - \lambda \log\int q_g(\theta) \rd\theta,
\end{equation*}
where $q_g(\theta) = \exp\left( - \frac{1}{\lambda} \left( \frac{1}{n}\sum_{i=1}^n h_{\theta}(x_i)g_i + \lambda'\|\theta\|_2^2 \right) \right)$.

The Fenchel conjugate $\ell_i^*$ can be explicitly described for typical loss functions. For instance, for the squared loss function $\ell_i(z) = 0.5(z-y_i)^2$, its Fenchel conjugate is $\ell_i^*(g) = 0.5 g^2 + g y_i$.

Direct computation of the normalization term of $q_g$ is difficult in general, but it can be efficiently estimated by the following procedure.
First, we reformulate this term as follows:
\begin{align*}
    \int q_g(\theta) \rd\theta
    &= \int \exp\left( - \frac{1}{\lambda} \left( \frac{1}{n}\sum_{i=1}^n h_{\theta}(x_i)g_i + \lambda'\|\theta\|_2^2 \right) \right) \rd\theta \\
    &= Z \int \exp\left( - \frac{1}{\lambda n}\sum_{i=1}^n h_{\theta}(x_i)g_i \right) \frac{\exp\left( -\frac{\lambda'}{\lambda}\|\theta\|_2^2 \right)}{Z} \rd\theta,
\end{align*}
where $Z$ is the normalization term of the Gaussian distribution appearing above, that is, $Z = \int \exp\left( -\frac{\lambda'}{\lambda}\|\theta\|_2^2 \right) \rd\theta$. 
Hence, we can estimate the normalization term of $q_g$ by the expectation with respect to the Gaussian distribution and computation of $Z$.
This expectation can be approximated by obtaining samples from the Gaussian distribution in proportion to $\exp\left( -\frac{\lambda'}{\lambda}\|\theta\|_2^2\right)$, and its normalization term is exactly $Z = \left( \frac{\pi \lambda}{\lambda'} \right)^{d/2}$. 
We can use the same samples from the Gaussian distribution for computing this expectation to reduce the computational cost.
Note that the normalization constant of the proximal Gibbs distribution $p_q$ can be approximated via the same procedure. 

\newpage

}

\end{document}